\documentclass[onecolumn, 11pt]{IEEEtran}
\usepackage{graphicx}
\usepackage{cite}
\usepackage{amsmath}
\usepackage{mathtools}
\usepackage{amsfonts, amssymb}
\usepackage{amsthm}
\usepackage{algorithm}
\usepackage{algorithmic}

\usepackage[mathscr]{eucal}

\usepackage{amsbsy}

\usepackage{bm}

\newcommand{\Sec}[1]{\hyperref[sec:#1]{\S\ref*{sec:#1}}} 
\newcommand{\Eqn}[1]{\hyperref[eq:#1]{(\ref*{eq:#1})}} 
\newcommand{\Fig}[1]{\hyperref[fig:#1]{Figure~\ref*{fig:#1}}} 
\newcommand{\Tab}[1]{\hyperref[tab:#1]{Table~\ref*{tab:#1}}} 
\newcommand{\Thm}[1]{\hyperref[thm:#1]{Theorem~\ref*{thm:#1}}} 
\newcommand{\Lem}[1]{\hyperref[lem:#1]{Lemma~\ref*{lem:#1}}} 
\newcommand{\Prop}[1]{\hyperref[prop:#1]{Property~\ref*{prop:#1}}} 
\newcommand{\Cor}[1]{\hyperref[cor:#1]{Corollary~\ref*{cor:#1}}} 
\newcommand{\Def}[1]{\hyperref[def:#1]{Definition~\ref*{def:#1}}} 
\newcommand{\Alg}[1]{\hyperref[alg:#1]{Algorithm~\ref*{alg:#1}}} 
\newcommand{\Ex}[1]{\hyperref[ex:#1]{Example~\ref*{ex:#1}}} 

\newcommand{\Real}{\mathbb{R}}

\newcommand{\V}[1]{{\bm{\mathbf{\MakeLowercase{#1}}}}} 

\newcommand{\M}[1]{{\bm{\mathbf{\MakeUppercase{#1}}}}} 

\newtheorem{mydef}{Definition}
\newtheorem{mypropos}{Proposition}
\newtheorem*{myremark}{Remark}
\newtheorem{mylemma}{Lemma}

\newtheorem*{myproof1}{Proof of Proposition 1}

\begin{document}

\title{Robust Large-Scale Non-Negative Matrix Factorization Using Proximal Point Algorithm}

\author{\IEEEauthorblockN{Jason Gejie Liu and Shuchin Aeron}\\
\IEEEauthorblockA{Department of Electrical and Computer Engineering \\ 
Tufts University, Medford, MA 02155\\
Gejie.Liu@tufts.edu, shuchin@ece.tufts.edu}
}

\maketitle

\begin{abstract}
A robust algorithm for non-negative matrix factorization (NMF) is presented in this paper with the purpose of dealing with large-scale data, where the separability assumption is satisfied. In particular, we modify the Linear Programming (LP) algorithm of \cite{Ref6} by introducing a reduced set of constraints for exact NMF. In contrast to the previous approaches, the proposed algorithm does not require the knowledge of factorization rank (extreme rays\cite{Ref5} or topics \cite{Ref_ex_2}). Furthermore, motivated by a similar problem arising in the context of metabolic network analysis\cite{Ref_ex_7}, we consider an entirely different regime where the number of extreme rays or topics can be much larger than the dimension of the data vectors. The performance of the algorithm for different synthetic data sets are provided.
\end{abstract}

\IEEEpeerreviewmaketitle

\section{Introduction}
\label{Sec1}

Matrix factorization has numerous applications to the real-world problems where the data matrices representing the numerical observations are often huge and hard to analyze. Meanwhile, factorizing them into lower-rank forms is able to reveal the inherent structure and features, which helps in the meaningful interpretation of the data. In a wide range of natural signals, such as pixel intensities, amplitude spectra, and occurance counts, negative values are usually physically meaningless. In order to deal with this non-negative constraint, Non-Negative Matrix Factorization (NMF) was introduced. 

NMF was first proposed in \cite{Ref2} and used by Lee and Seung for parts-based data representation \cite{Ref4}. It is well known that NMF may not be unique. In this context a sufficient condition on the uniqueness of NMF was pointed out in \cite{Ref5}. A geometrical interpretation \cite{Ref7}, of this condition amounts to the fact that the extreme rays (topics) generating the cone (in the non-negative orthant) are contained in the data. Thus, for NMF, one only needs to identify these extreme rays. Additionally, it was demonstrated in \cite{Ref_ex_1} that under such a separability assumption, one can use Linear Programming (LP) to isolate the extreme rays from the non-extreme rays. In this paper we will focus only on such cases.

Bittorf et al.\cite{Ref6} presented a LP-based NMF algorithm named \emph{Hottopixx}. Kumar et al.\cite{Ref7}, instead, presented a fast conical hull algorithm to deal with the large-scale NMF based on its polyhedral structure. It was shown to perform much faster than \emph{Hottopixx}. However, both of these algorithms require the factorization rank, i.e. the number of extreme rays as a necessary input. Some applications will grant this as prior knowledge but from the view of robustness, the preference would go to a robust NMF algorithm. Gillis and Luce \cite{Ref8} reformulated the algorithm in \cite{Ref6} to detect the number of extreme rays automatically. Nevertheless, the limitation still exists with the fact that the number of constraints in the LP is enormous in face of the large-scale data.

Alternatively, motivated by \textbf{Theorem 5.4} in\cite{Ref_ex_1}, we propose a simpler modification of the constraints to alleviate these issues. In particular we reduce the number of constraints and do not require that the number of extreme rays to be known. This allows us to use a proximal point-based algorithm \cite{Ref10} to solve the LP problem efficiently for data sets with large size. 

In addition we also consider an entirely different regime from the NMF applications in the literature so far, where the data lies high dimension with a much smaller factorization rank (the number of extreme rays) in comparison. Specifically, we look at the case when the number of extreme rays is much larger than the dimensionality of the space. This is caused by the computational issues, which arises in Double Description (DD) method \cite{Ref_ex_6} for the analysis of metabolic networks to find the Elementary Flux Modes (EFMs) as the set of extreme rays of the polyhedral cone \cite{Ref_ex_7}. A NMF like problem arises as an intermediate step in DD method. In this context, we believe that the computational advances in NMF can help with addressing the computational issues in the DD method \cite{Ref_ex_6}. 

The organization of the paper is as follows. Section \ref{Sec2} provides a brief review of NMF from the geometric perspective as well as the \emph{Hottopixx} algorithm. Section \ref{Sec3} explains the proposed proximal point algorithm with the reformulated LP constraints. The experiments results are presented in Section IV and the paper concludes in Section V.

\textbf{Notation}: The matrices will be denoted by boldface capital letters and vectors by boldface small letters. In addition we use the MATLAB notation of $\text{diag}$ to transform vectors to diagonal matrices and to extract the diagonal from the matrix in the argument. Also we use the MATLAB $``,"$ and the $``;"$ operators for matrix concatenations.

\section{Non-negative Matrix Factorization}
\label{Sec2}
For the non-noisy case, given a data matrix $\M{X} = [\V{x}_1,\V{x}_2,...,\V{x}_n] \in \mathbb{R}_{+}^{m\times n}$. Therefore, NMF aims to find two nonnegative matrices $\M{F}\in\mathbb{R}_{+}^{m\times r}$ and $\M{W}\in\mathbb{R}_{+}^{r\times n}$  such that $\M{X} = \M{F}\M{W}$. For an approximate NMF, instead, the aim is to solve the following optimization problem. 
\begin{align}
\min_{ \M{F}, \M{W} \geq 0} || \M{X} - \M{F} \M{W}||_{2}^{2}
\end{align}

\subsection{Geometry of the NMF Problem}
The factorization $\M{X} = \M{F}\M{W}$ implies that all the \emph{columns} of $\M{X}$ can be represented as non-negative combination of the columns $\{\V{f}_i\}_{i = 1}^{r}$ of the matrix $\M{F}$. The algebraic characterization can be described as below.
\begin{mydef}
The $\mathbf{simplicial~cone}$ generated by columns  $ \{\V{f}_i\}_{i = 1}^{r}$ is given by,
\begin{equation} 
\Gamma = \Gamma_{F} = \{ \V{x}: \V{x} = \sum_{i}\alpha_{i} \V{f}_{i},\alpha_{i} \geq 0 \}
\end{equation}
\end{mydef} The factorization $\M{X} = \M{F}\M{W}$ refers  geometrically to that the $\V{x}_i, i = 1,2,...,n$ all lie in or on the surface of the simplical cone generated by the $\{\V{f}_i\}_{i=1}^{r}$, as depicted in Fig.~\ref{fig1}.
\begin{figure}[ht]
\centering
\includegraphics[width=0.28\textwidth]{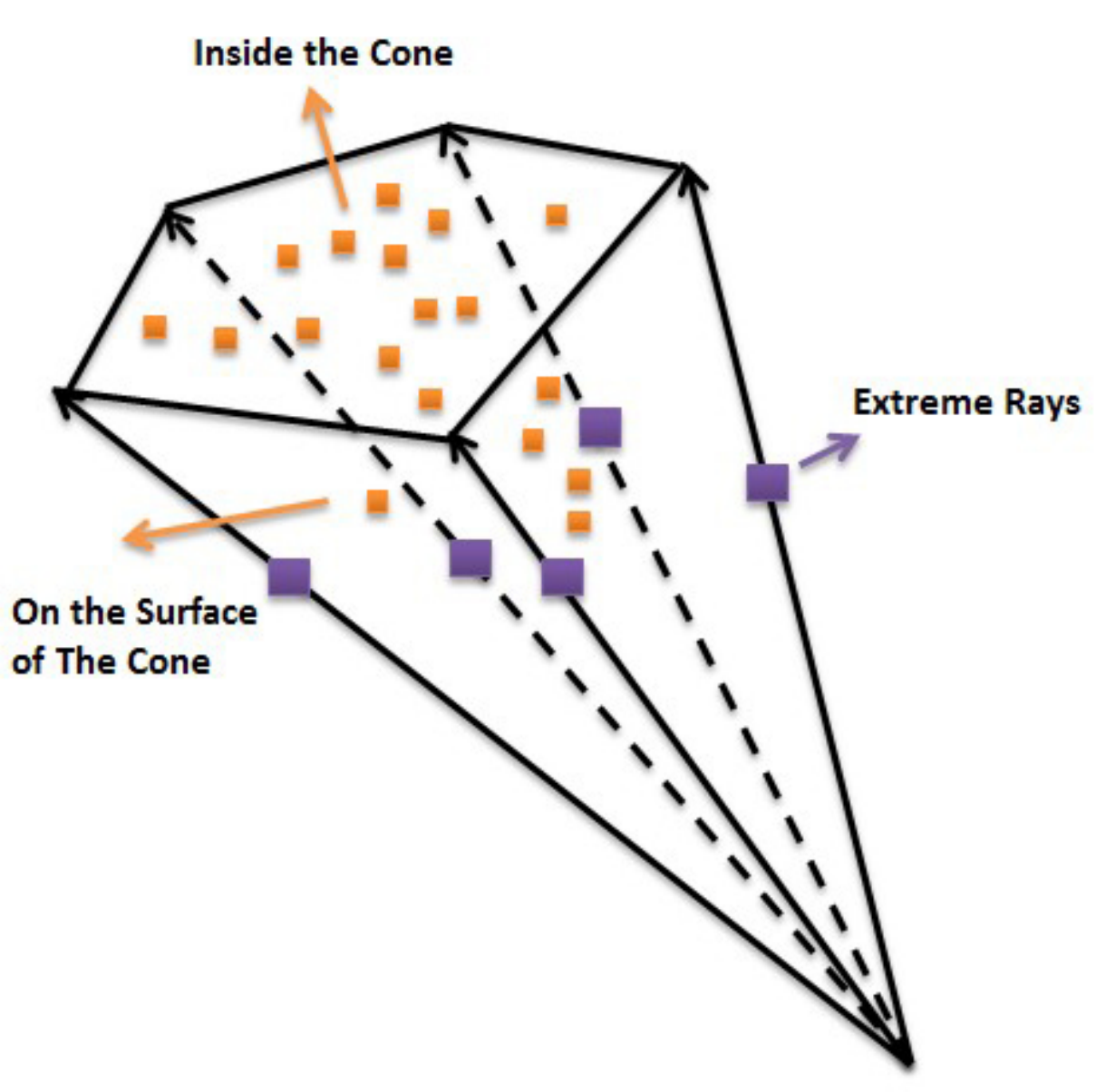}
\caption{Geometry of the NMF Problem. Separability implies that data is contained in a cone generated by a subset of $r$ extreme rays (indicated by purple squares).}
\label{fig1}
\end{figure} 

With this viewpoint in mind, we define three assumptions as follows.
\begin{itemize}
\item \textbf{Assumption 1}: \emph{Extreme rays} by definition are simplicial: No extreme ray is in the convex combination of the other extreme rays. This is also shown to be necessary and sufficient for exact recovery in topic modeling \cite{Ref_ex_9}.
\item \textbf{Assumption 2}: The dataset consisting of all columns of $\M{X}$, reside in or on the surface of a cone generated by these extreme rays of $\M{X}$\cite{Ref5}.
\item \textbf{Assumption 3}: Assuming that the columns of $\M{X}$ are normalized to unity there are no duplicate columns in $\M{X}$.
\end{itemize}
 
The above three assumptions will be collectively referred to as \emph{separability assumption} in the following: The entire dataset, i.e. all columns of $\M{X}$, reside in or on a surface of a cone generated by a small subset of $r$ columns of $\M{X}$, the vectors in this subset being simplicial and there are no duplicate columns in $\M{X}$ after column normalization.

In algebraic terms, $\M{X} = \M{F}\M{W} = \M{X}_{I}\M{W}$ for some subset $I \subseteq \{1,2,...,r\}$ of columns (\textbf{extreme rays}) of $\M{X}$ and where $\M{X}_{I}$ denotes the matrix built with columns of $\M{X}$ indexed by $I$. This means that the $r$ vectors of $\M{F}$ are hidden among the columns of $\M{X}$ ($I$ is unknown) \cite{Ref_ex_1}. Equivalently, it implies that the corresponding subset of $r$ \emph{rows} of $\M{W}$ constitutes the $r \times n$ weight matrix. Therefore, the computational challenge is to identify the extreme rays efficiently. In this context, we first outline the LP-based \emph{Hottopixx} Algorithm from \cite{Ref6} .
\subsection{Hottopixx}
Bittorf et al.\cite{Ref6} proposed an algorithm of NMF under separability assumption based on the following LP problem:
\begin{equation}
\begin{aligned}
\label{eq2_2}
& \min_{\M{C}\in\Phi_1(\M{X})} \V{p}^{T}\text{diag}(\M{C}) \\
\end{aligned}
\end{equation} and $\V{p}\in\Real^{n\times 1}$ is a random vector with distinct positive entries and $\M{C}\in\mathbb{R}_{+}^{n\times n}$ is referred to as a factorization localizing matrix\cite{Ref6}, which belongs to the following polyhedral set. 
\begin{align}
\label{eq2_3}
\Phi_1(\M{X}) =  \{ & \M{C}:  \M{X}\M{C} = \M{X}, \mbox{Trace}(\M{C}) = r, \M{C}(i,i)\leq 1~\text{for~all}~i \nonumber \\
                               & \M{C}(i,j)\leq \M{C}(i,i)~\text{for~all}~i,~j, \M{C} \geq 0 \}
\end{align}
For a large scale set-up they proposed an incremental gradient descent algorithm to solve the LP.

\section{Robust NMF Using \emph{Proximal Point} Algorithm}
\label{Sec3}
As explained before, two of the prominent shortcomings of existing algorithms for the NMF problem are - (i) Dependence on knowledge of the number of extreme rays $r$ and, (ii) Dealing with a large data set resulting in an enormous number of constraints. An approach in this direction was taken in \cite{Ref8}. However, the number of constraints in their reformulation is still immense for large data. 
In this paper we present a reformulation which drastically reduces the set of constraints. 

\subsection{LP Reformulation}
Assuming that the columns of $\M{X}$ are normalized to have an unit $\M{L}_1$ norm, our LP reformulation for NMF is given as
\begin{align}
\label{eq3_1}
& \min_{\M{C}\in \Phi_2(\M{X})}  \V{p}^{T}\text{diag}(\M{C})
\end{align}
where, 
\begin{align}
\label{eq3_4}
& \Phi_2(\M{X}) = \{ \M{C}: \M{X}\M{C} = \M{X},  \M{C}^T \V{1} = \V{1}, \M{C} \geq 0\}
\end{align} where $\V{1} \in\Real_{+}^{n\times 1}$ is the vector of all $1$-s and $\V{p}\in\Real_{+}^{n\times 1}$ is the same as the vector in (\ref{eq2_2}).
\begin{mypropos}
\label{eq3_5}
Suppose $\M{X}$ admits a separable factorization $\M{F} \M{W}$, compute the minimization of (\ref{eq3_1}) and let $I = \{i: \M{C}_{ii} = 1\}$, then $\M{F} = \M{X}_{I}$.
\end{mypropos} 

In order to prove the above proposition, we consider the Lagrangian of the optimization problem in (\ref{eq3_1}), which is,
\begin{equation}
\begin{aligned}
\label{eq_5}
L(\M{C},\M{R},\M{\lambda}) = \underset{\M{C}}{\text{min}}~& \M{p}^{T}\text{diag}(\M{C}) + \text{Tr}\{\M{R}^{T}(\M{X}\M{C} - \M{X})\} \\
& + \M{\lambda}^{T}(\M{C}^{T} \textbf{1} - \textbf{1}) + \text{Tr}\{\M{M}^{T}\M{C}\}
\end{aligned}
\end{equation} where $\M{R},\M{\lambda}$ and $\M{M}$ are the Lagrange multipliers. Then the dual form of (\ref{eq3_1}) is


\begin{equation}
\begin{aligned}
\label{eq_9}
& \max_{\M{R},\M{\lambda},\M{M}} & & -\text{Tr}\{\M{R}^{T}\M{X}\} - \M{\lambda}^{T}\textbf{1} \\
& \text{s.j.t} & & \M{P} + \M{X}^{T}\M{R} + \textbf{1}\M{\lambda}^T + \M{M} = 0,\M{M}\geq 0
\end{aligned}
\end{equation} The proof of the proposition follows from \textbf{Lemma 1} and \textbf{Lemma 2} below. 
\begin{mylemma}
If $\ell\notin I, \M{C}_{\ell\ell} = 0~for~all~\M{C}\in\Phi_2(\M{X})$.
\end{mylemma}
\begin{proof}
For $\ell\notin I$, consider the LP problem
\begin{equation}
\begin{aligned}
\label{eq_10}
& \underset{\M{C}\in\Phi_2(X)}{\text{min}} & -\V{e_{\ell}}^{T}\text{diag}(\M{C})
\end{aligned}
\end{equation} where $\V{e_{\ell}}\in\Real^{n\times 1}_{+}$ denotes the vector with $\ell$th entry $1$ and the rest $0$. Assign $\M{P} = -\text{diag}(\V{e_{\ell}})$ and using the constraint $\M{C}\geq 0$, we can claim that $-\V{e}_{\ell}^{T}\text{diag}(\M{C}) \leq 0$. Under the separability assumption, there exists a collection of vectors $\{\V{\rho}_i\}\in\Real^{n\times 1}$ such that,
\begin{eqnarray}
\label{eq_12}
\rho_i^T \V{x}_{i} & = u_i & \\ 
\rho_i ^T \V{x}_{j} & \leq -v_i,& ~\text{for}~j\neq i\nonumber
\end{eqnarray} for $u_i = 0$ and some $v_i\geq 0$. A feasible solution to (\ref{eq_9}) is 
\begin{equation}
\begin{aligned}
\label{eq_21}
& \M{P} = -\text{diag}(\V{e_{\ell}}),~\V{\lambda} = \V{0}\in\Real^{n\times 1} \\
& \M{R} = [0,...,\V{\rho}_{i},...,0], \,\, \mbox{for some}\,\, i\in I \\\
& \M{M} = \M{M}_1 + \M{M}_2: \M{M}_1 = \text{diag}(\V{e}_{\ell}),~ \M{M}_2 = -\M{X}^{T}\M{R} = \M{0}
\end{aligned} 
\end{equation} With such selection, the dual cost function is equal to  $0$. From \emph{weak duality} \cite{Ref_ex_5} it follows that $-\V{e}_{\ell}^{T}\text{diag}(\M{C}) \geq 0$. Combined with $-\V{e}_{\ell}^{T}\text{diag}(\M{C}) \leq 0$, it implies $\V{e}_{\ell}^{T}\text{diag}(\M{C}) = 0$ and $\M{C}_{\ell \ell} = 0$.
\end{proof}

\begin{mylemma}
\label{lem2}
If $\ell\in I, \M{C}_{\ell\ell} = 1~for~all~\M{C}\in\Phi_2(\M{X})$.
\end{mylemma}
\begin{proof}
For $\ell\in I$, Consider the LP problem
\begin{equation}
\begin{aligned}
\label{eq_11}
& \underset{\M{C}\in\Phi_2(\M{X})}{\text{min}} & & \V{e}_{\ell}^{T}\text{diag}(\M{C}) 
\end{aligned}
\end{equation} Note that the constraint $\M{C}^{T}\textbf{1} = \textbf{1}$ implies that $\M{C}\leq 1$ therefore $\V{e}_{\ell}^{T}\text{diag}(\M{C}) \leq 1$. For the dual program a feasible solution can be found as
\begin{equation}
\begin{aligned}
\label{eq_13}
& \M{P} = \text{diag}(\V{e}_{\ell}),\V{\lambda}^{T} = [0,0,... -1,...,0]\in\Real^{1\times n}, \text{where}~\ell\text{th entry is} -1 \\
& \M{R} = 0,~ \M{M} = -\V{1\lambda}^{T} - \M{P} 
\end{aligned}
\end{equation} for which the dual cost function (\ref{eq_9}) is equal to 1. Again using the \emph{weak duality} \cite{Ref_ex_5}, it implies that $\V{e}_{\ell}^{T}\text{diag}(\M{C}) \geq 1$ and from $\V{e}_{\ell}^{T}\text{diag}(\M{C}) \leq 1$, we have $\V{e}_{\ell}^{T}\text{diag}(\M{C}) =1$ and $\M{C}_{\ell \ell} = 1$. 
\end{proof} 

\begin{myproof1}
\emph{Let $\M{C}_0$ denote the factorization localizing matrix which identifies the factorization with lowest cost
$\V{p}^{T}\text{diag}(\M{C})$ and has either ones or zeros on the diagonal. Then $\M{C}_0$ is the unique optimal solution of (\ref{eq3_1}). To see this, let $I$ denote the set of simplicial columns of minimum cost. Since each column of $\M{X}$ can only belong to $I$ or not, once $\M{X}$ is given, $\M{C}$ is determined then the lowest cost $\V{p}^{T}\text{diag}(\M{C})$ is determined, which is unique.
}
\end{myproof1}

\begin{myremark}
Note that $\M{W}$ can be readily obtained from $\M{C}$ as $\M{W} = \M{C}(I,:)$ (in MATLAB notation). 
\end{myremark}

\subsection{Proximal Point Algorithm}
Based on the above reformulation, a proximal point-based (\emph{Proximal Point}) algorithm is employed in this paper. A necessary pre-processing step is the column normalization, which makes sure that the sum of each column of $\M{X}$ is equal to one.

After the normalization, we can rewrite the LP in (\ref{eq3_1}) as 
\begin{align}
\label{eq4_1}
& \min_{\M{C} \in \Phi_3(\M{X})}  \V{p}^{T}\text{diag}(\M{C})
\end{align}
where, 
\begin{align}
\label{eq4_2}
\Phi_3(\M{X}) = \{ \M{C}: \M{AC} = \M{A}, \M{C} \geq 0 \}
\end{align}
where $\M{A} = [\M{X};\V{1}^T]$

%
%
We solve this LP using the proximal point algorithm \cite{Ref10} in \textbf{Algorithm 1} and in our implementation, $\{t_k\}$ are set to a large constant $100$. The discussion on the convergence of the algorithm can be found in \cite{Ref10}.
\begin{algorithm}
\textbf{Input:} A column normalized matrix $\M{X} \in \Real_{+}^{m\times n}$, stopping threshold $\epsilon$.\\
\textbf{Output:} A matrix $\M{F}\in\Real_{+}^{m\times r}$ and $\M{W}\in\Real_{+}^{r\times n}$, and $\M{X} = \M{FW}$.
\begin{algorithmic}
\item[1:] Initialize $\M{Q}^{0} = 0$ and $\M{C}^{0} = 0$, randomly generate $\V{p}\in \Real_{+}^{m\times 1}$. \\
\item[2:] Update $\M{C}^{k+1}$: 
\begin{equation}
\label{eq_20}
\begin{aligned}
\M{C}^{k+1} & = \underset{\mathbf{C}} {\text{argmin}}\{\V{p}^{T}\text{diag}(\M{C}) + \frac{1}{t^{k}}||\mathbf{Q}^{k}+t^{k}(\M{A}\M{C} - \M{A})||^2_2 \} \\ 
& = \frac{1}{2t^{k}}(\M{A^{T}A})^{-1}\left(2t^{k} \M{A}^{T} \M{A}- \text{diag}(\V{p}) - 2\M{A}^{T}\mathbf{Q}^{k}\right) \nonumber  
\end{aligned}
\end{equation}
\item[3:] Project $\M{C}^{k+1}$ to the constraint $\M{C}\geq 0$ using $\M{C}^{k+1} = \mathrm{pos}(\M{C}^{k+1})$, where $\mathrm{pos}(\cdot)$ keeps the positive elements and switch the negative elements to $0$.
\item[4:] Update $\mathbf{Q}^{k+1}$ : $\mathbf{Q}^{k+1} = \mathbf{Q}^{k} + t^{k}(\M{AC}^{k+1} - \M{A}) $.
\item[5:] Stop the iterations if $||\M{C}^{k+1} - \M{C}^{k}||_2\leq\epsilon$.
\item[6:] Let $I = \{i:\M{C}_{ii}=1\}$ and set $\M{F} = \M{X}_{I}$ as well as obtain $\M{W} = \M{C}(I,:)$.  
\end{algorithmic}
\caption{Robust NMF by \emph{Proximal Point} Algorithm}
\label{algo:relgraph}
\end{algorithm} 

%
\section{Experiments Results}
\label{Sec4}
All of the experiments were run on an identical configuration: a dual Xeon W3505 (2.53GHz) machine with 6GB RAM. \emph{Proximal Point} Algorithm is examined in MATLAB with the version of 2013a. 

\subsection{Random Data Generation}
To generate our instances, $r$ independent extreme rays are firstly created in $\mathbb{R}_{+}^{m\times 1}$, with the element value between $[0,100]$. The remaining columns are then generated to be the random non-negative combinations of the $r^{\prime}$ extreme rays, where $r^{\prime}\in[2,r]$ is randomly selected for each of the $n-r$ points. The column normalization is carried out sequentially. Three regimes of NMF problems are analyzed here: 
\begin{itemize}
\item (C1). $m \geq n, m \geq r$, which is motivated from the data structure for face recognition\cite{Ref_ex_8} \item (C2). $r\leq m\leq n$, which is the scenario for topic modeling problem \cite{Ref_ex_2}
\item (C3). $m\leq r\leq n$, which can be applied to metabolic network data \cite{Ref_ex_7}.
\end{itemize}

Furthermore, since the algorithm is free from the order of the columns, the $r$ extreme rays are allocated at the beginning of each data set. 

Different size of data sets are generated to check the effectiveness of \emph{Proximal Point} Algorithm, from small to large-scale. In Tab.~\ref{table_1}, the last column indicates the highest level for iteration stopping criterion $\epsilon$ to achieve the listed accuracy. From the experiments, it is exhibited that our algorithm can deal with three regimes of the data with different sizes. Moreover, the identification accuracy is satisfying.  

\begin{table}[!t]
\renewcommand{\arraystretch}{1.3}
\caption{Experiments on different Dataset regarding to C1-C3}
\label{table_1}
\centering
\begin{tabular}{|l||c|c|c|}
\hline
Data Set & \# of Extreme Rays & Accuracy & $\epsilon$ \\
\hline
$100\times 75$(C1) & 25 & $25/25$ & $10^{-5}$\\
\hline
$500\times 375$(C1) & 25 & $23/25$ & $10^{-4}$\\
\hline
$1200\times 600$(C1) & 300 & $300/300$ & $10^{-4}$ \\
\hline
$25\times 100$(C2) & 15 & $14/15$ & $10^{-5}$ \\
\hline
$125\times 500$(C2) & 75 & $74/75$ & $10^{-4}$\\
\hline
$425\times 1200$(C2) & 225 & $223/225$ & $10^{-4}$\\
\hline
$25\times 100$ (C3) & 45 & $45/45$ & $10^{-5}$ \\
\hline
$125\times 500$ (C3) & 150 & $150/150$ & $10^{-4}$ \\
\hline
$425\times 1200$(C3) & 625 & $625/625$ & $10^{-4}$\\
\hline
\end{tabular}
\end{table}

\subsection{Application to Image Processing}
In this section, we apply the \emph{Proximal Point} algorithm to one face image processing data set, namely, CBCL Dataset\cite{Ref_ex_8}. Basically, the CBCL face dataset is made of 2429 gray-level face images with $19\times 19$ pixels. We randomly choose $20$ images from the dataset with vectorization to be the generators, which means the number of extreme rays in this case is $r=20$. Through the random non-negative combination of the extreme rays, a $361\times 500$ facial data matrix is created. Applying \emph{Proximal Point} algorithm to this dataset, the results of the extreme rays identification are shown in Fig.~\ref{fig3}, which represents the initial $20$ images as generators. 
\begin{figure}[ht]
\centering
\includegraphics[scale=0.55]{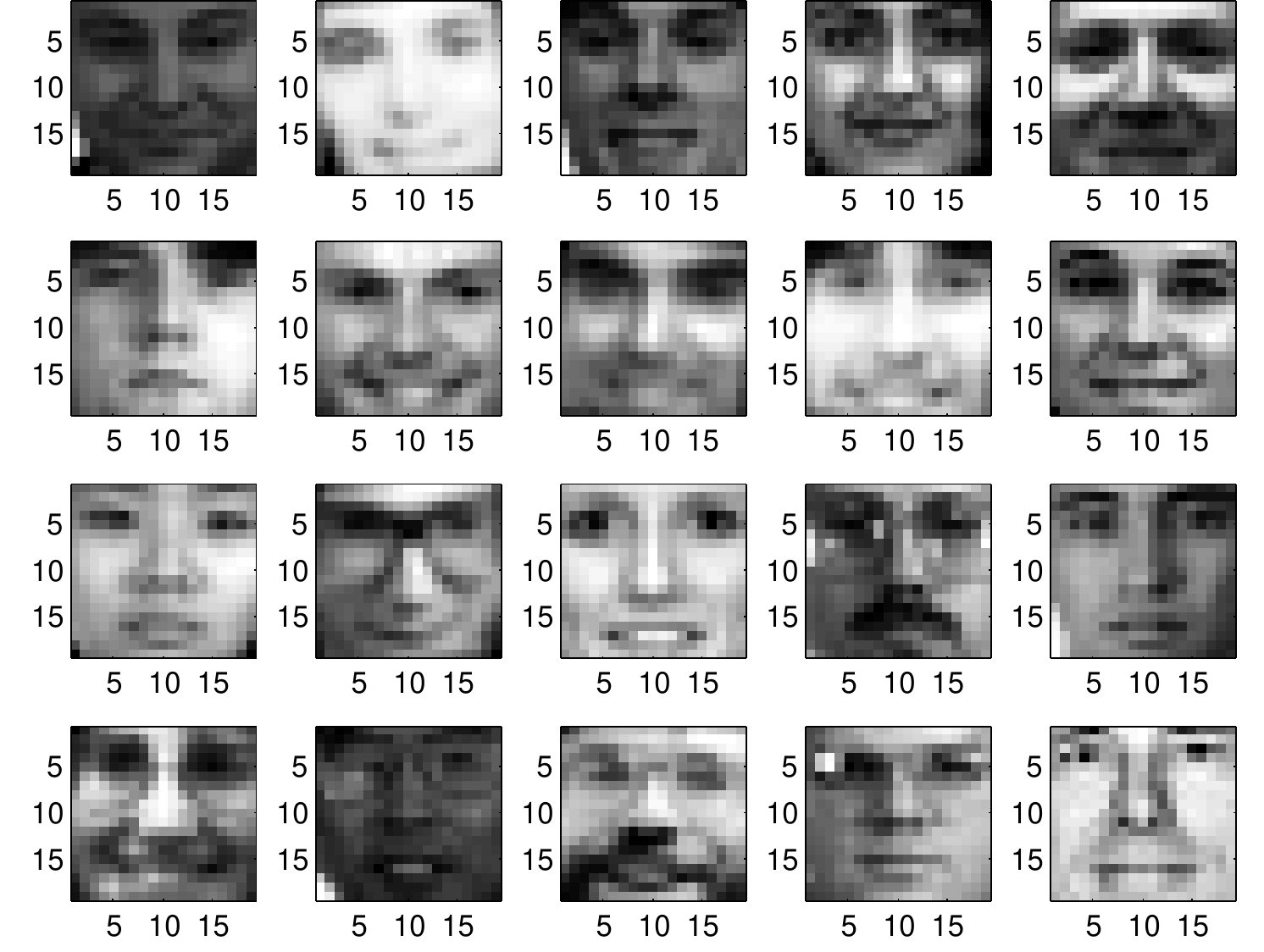}
\caption{The facial images identified as extreme rays for the $361\times 500$ data set with $20$ extreme rays. The stopping criterion was selected as $10^{-5}$.}
\label{fig3}
\end{figure}

\section{Acknowledgements}
The second author would like to acknowledge several useful discussions with Prof. Prakash Ishwar at Boston University.

\end{document}